%% file: arxiv-real-compression.tex
\documentclass{article}
\usepackage[preprint]{my_nips_2018}

\usepackage[utf8]{inputenc} % allow utf-8 input
\usepackage[T1]{fontenc}    % use 8-bit T1 fonts
\usepackage{hyperref}       % hyperlinks
\usepackage{url}            % simple URL typesetting
\usepackage{booktabs}       % professional-quality tables
\usepackage{amsfonts}       % blackboard math symbols
\usepackage{nicefrac}       % compact symbols for 1/2, etc.
\usepackage{microtype}      % microtypography

\usepackage{mathtools}
\usepackage{amsmath}
\usepackage{amsthm}
\usepackage{amssymb}

\usepackage{todonotes}

\usepackage[ruled]{algorithm2e}
\usepackage{algorithmic}
\usepackage{stmaryrd}  %for misc. logical symbols

\newcommand{\nrm}[1]{\left\Vert #1 \right\Vert}

\newcommand{\R}{\mathbb{R}}

\newcommand{\N}{\mathbb{N}}

\newcommand{\paren}[1]{\left( #1 \right)}

\newcommand{\set}[1]{\left\{ #1 \right\}}

\newcommand{\beq}{\begin{eqnarray*}}
\newcommand{\eeq}{\end{eqnarray*}}
\newcommand{\beqn}{\begin{eqnarray}}
\newcommand{\eeqn}{\end{eqnarray}}
\newcommand{\ben}{\begin{enumerate}}
\newcommand{\een}{\end{enumerate}}
\newcommand{\bit}{\begin{itemize}}
\newcommand{\eit}{\end{itemize}}
\newcommand{\hide}[1]{}

\newcommand{\eps}{\varepsilon}

\newcommand{\ceil}[1]{\ensuremath{\left\lceil#1\right\rceil}}
\newcommand{\floor}[1]{\ensuremath{\left\lfloor#1\right\rfloor}}

\newcommand{\diam}{\operatorname{diam}}
\newcommand{\ddim}{\operatorname{ddim}}

\newcommand{\gn}{\, | \,}

\newcommand{\calA}{\mathcal{A}}

\newcommand{\calE}{\mathcal{E}}

\newcommand{\calL}{\mathcal{L}}

\renewcommand{\P}{\mathbb{P}}
\newcommand{\E}{\mathbb{E}}
\newcommand{\adv}{\gamma}
\newcommand{\dev}{\eta}
\newcommand{\tlearn}{T_{\calA}}
\newcommand{\teval}{T_{\calE}}

\newcommand{\F}{\mathcal{F}}

\newcommand{\C}{\mathcal{C}}
\newcommand{\X}{\mathcal{X}}
\newcommand{\Y}{\mathcal{Y}}
\newcommand{\K}{\mathcal{K}}

\DeclarePairedDelimiter\abs{\lvert}{\rvert}%

\newcommand{\bigO}[1]{{O}\left(#1\right)}

\newcommand{\uf}[1]{\frac{1}{#1}}

\newcommand{\algname}[1]{{\tt#1}}

\newcommand{\ind}{\mathbb{I}}
\newcommand{\bxi}{\boldsymbol{\xi}}
\newtheorem{theorem}{Theorem}
\newtheorem{definition}[theorem]{Definition}
\newtheorem{lemma}[theorem]{Lemma}
\newtheorem{corollary}[theorem]{Corollary}

\title{Sample Compression for Real-Valued Learners}

\author{
  Steve Hanneke \\
  Princeton, NJ \\
  \texttt{steve.hanneke@gmail.com}
  \And
  Aryeh Kontorovich \\
  Ben-Gurion University\\
  \texttt{karyeh@bgu.sc.il} 
  \And
  Menachem Sadigurschi \\
  Ben-Gurion University \\
  \texttt{sadigurs@post.bgu.ac.il}
}

\begin{document}

\maketitle

\begin{abstract}
  We give an algorithmically efficient version
  of the learner-to-compression scheme conversion in
  Moran and Yehudayoff (2016).
  In extending this technique to real-valued hypotheses,
  we also obtain an efficient
  regression-to-bounded sample compression
  % conversion from regression to a
  % bounded sample compression scheme
  converter.
  To our knowledge, this is the first general 
  compressed regression result 
  (regardless of efficiency or boundedness)
  guaranteeing uniform approximate reconstruction.
  Along the way, we develop
  a generic procedure for constructing weak real-valued learners
  out of abstract regressors; this may be of independent interest.
  In particular, this
  result
  sheds new light on
  %is relevant to
  an 
  open question of H. Simon (1997).
  We show applications to two regression problems:
  learning Lipschitz and bounded-variation functions.
\end{abstract}

\section{Introduction}
\label{sec:intro}

\begin{sloppypar}
Sample compression 
is a natural learning strategy,
%in which
whereby
%one
the learner
seeks to retain a small subset of the training examples,
which (if successful)
may then be decoded as
%will
%%that
%unambiguously define
a hypothesis with low empirical error.
Overfitting is controlled by the size of this
learner-selected
``compression set''.
%This is
%a special case of the more general
Part of a
%larger
more general
{\em Occam learning} paradigm,
such results are commonly summarized by ``compression implies learning''.
%The various manifestation
%Colloquially, this phenomenon (in its various 
A fundamental question,
posed by \citet{Littlestone86relatingdata},
concerns the reverse implication:
Can every learner be converted into a sample compression scheme?
Or, in a more quantitative formulation:
Does every VC class admit a constant-size
sample compression scheme?
A series of partial results
\citep{floyd1989space,helmbold1992learning,DBLP:journals/ml/FloydW95,ben1998combinatorial,DBLP:journals/jmlr/KuzminW07,DBLP:journals/jcss/RubinsteinBR09,DBLP:journals/jmlr/RubinsteinR12,MR3047077,livni2013honest,moran2017teaching}
culminated in \citet{DBLP:journals/jacm/MoranY16}
which resolved the latter question\footnote{
  The refined conjecture of \citet{Littlestone86relatingdata}, that any concept class with VC-dimension $d$
  admits a compression scheme of size $O(d)$, remains open.}.
\end{sloppypar}

\citeauthor{DBLP:journals/jacm/MoranY16}'s
solution involved a clever use of von Neumann's minimax theorem,
which allows one to
make the leap
from
the existence of a weak learner uniformly over all {\em distributions on examples}
to
the existence of a {\em distribution on weak hypotheses} under which they achieve
a certain performance simultaneously over all of the examples.
Although their paper can be understood without any knowledge of boosting,
\citeauthor{DBLP:journals/jacm/MoranY16} note the well-known connection between boosting
and compression. Indeed, boosting may be used to obtain a constructive proof
of the minimax theorem
\citep{freund1996game,MR1729311}
--- 
and this connection was what motivated us to
seek an efficient algorithm implementing
\citeauthor{DBLP:journals/jacm/MoranY16}'s
existence proof.
Having obtained an efficient conversion procedure
from consistent PAC learners to bounded-size sample
compression schemes, we turned our attention to the case of real-valued hypotheses.
It turned out that a virtually identical boosting framework
could be made to work for this case as well, although a novel analysis was required.

\paragraph{Our contribution.}
\begin{sloppypar}
Our point of departure is the simple but powerful observation
%\citep[Display (6.29)]{MR2920188}
\citep{MR2920188}
%that the $\alpha$-Boost algorithm
that many boosting algorithms (e.g., AdaBoost, $\alpha$-Boost) are capable of outputting 
%outputs 
a family of $O(\log(m)/\adv^2)$ hypotheses
such that not only does their (weighted) majority vote yield a sample-consistent
classifier, but in fact a $\approx(\frac12+\adv)$ super-majority does as well.
This fact implies that after boosting, we can sub-sample a constant (i.e., independent of sample size $m$)
number of classifiers and thereby efficiently recover the sample compression bounds
of 
\citet{DBLP:journals/jacm/MoranY16}.
\end{sloppypar}

Our chief technical contribution, however, is in the real-valued case.
As we discuss below, extending the boosting framework from classification to regression
presents a host of technical challenges, and there is currently no off-the-shelf
general-purpose analogue of AdaBoost for real-valued hypotheses.
One of our insights is to impose distinct error metrics on the weak and strong learners:
a ``stronger'' one on the latter and a ``weaker'' one on the former.
%Our innovation lies in imposing a ``strong'' error metric on the strong learner
%and a ``weaker'' one on the weak one.
This allows us to achieve
%several
two
goals simultaneously:
%In order of increasing importance:
\begin{itemize}
\hide{
\item[(a)]
\citet{kegl2003robust}'s
MedBoost
---
%  A known variant
%\citep{DBLP:conf/colt/Kegl03}
a variant of
%of essentially
the classical $\alpha$-Boost algorithm
---
achieves a weak-to-strong conversion, for the notions we define. The analysis is standard.
}
\item[(a)]
  We give apparently the first generic construction for our weak learner,
  demonstrating that the object is natural and abundantly available. This is in contrast with
  many previous proposed weak regressors, whose stringent or exotic definitions made them
  %difficult
  unwieldy
  to construct or verify as such. The construction is novel and may be of independent interest.
\item[(b)]
  We show that the output of a certain real-valued boosting algorithm may be sparsified so as to yield a constant size sample compression
  analogue of the
  \citeauthor{DBLP:journals/jacm/MoranY16} result for classification.
  This gives the first general constant-size sample compression scheme having uniform approximation guarantees on the data. %for generic real-valued learners.
\end{itemize}

%\paragraph{Related work.}

\section{Definitions and notation}
\label{sec:defnot}

We will write
$[k]:=\set{1,\ldots,k}$.
An {\em instance space} is an abstract set $\X$.
For a concept class $\C\subset\set{0,1}^\X$,
if say that $\C$ {\em shatters} a set $\set{x_1,\ldots,x_k}\subset\X$
if
$$\C(S) = \set{ (f(x_1),f(x_2),\ldots,f(x_k)) : f\in\C}
%\subseteq
=
\set{0,1}^k
.
$$
%has cardinality $2^k$.
The VC-dimension $d=d_\C$ of $\C$ is the size of the largest shattered set (or $\infty$ if $\C$ shatters
sets of arbitrary size) \citep{MR0288823}.
When the roles of $\X$ and $\C$ are exchanged ---
that is, an $x\in\X$ acts on $f\in\C$
%, which evaluates to
via
$x(f)=f(x)$, ---
we refer to $\X=\C^*$ as the {\em dual} class of $\C$.
Its VC-dimension is then $d^*=d^*_\C:=d_{\C^*}$, 
and referred to as the \emph{dual VC dimension}.
\citet{MR723955} showed that $d^*\le2^{d+1}$.

For $\F\subset\R^\X$ and $t>0$,
we say that $\F$ $t$-shatters
a set $\set{x_1,\ldots,x_k}\subset\X$
if
$$\F(S) = \set{ (f(x_1),f(x_2),\ldots,f(x_k)) : f\in\F}\subseteq\R^k$$
contains the translated cube $\set{-t,t}^k+r$ for some $r\in\R^k$.
The $t$-fat-shattering dimension $d(t)=d_\F(t)$
is the size of the largest $t$-shattered set (possibly $\infty$) \citep{alon97scalesensitive}.
Again, the roles of $\X$ and $\F$ may be switched,
in which case $\X=\F^*$ becomes the dual class of $\F$.
Its $t$-fat-shattering dimension is then $d^*(t)$,
and \citeauthor{MR723955}'s argument shows that $d^*(t)\le2^{d(t)+1}$.

A {\em sample compression scheme} $(\kappa,\rho)$ for a hypothesis
class $\F\subset\Y^\X$
%(either binary or real-valued)
is defined as follows.
A $k$-{\em compression} function $\kappa$
maps sequences $((x_1,y_1),\ldots,(x_m,y_m))\in\bigcup_{\ell\ge1}(\X\times\Y)^\ell$
to elements in
$\K=\bigcup_{\ell\le k'}(\X\times\Y)^\ell
\times
\bigcup_{\ell\le k''}\set{0,1}^\ell$,
where $k'+k''\le k$.
A {\em reconstruction} is a function $\rho:\K\to\Y^\X$.
We say that $(\kappa,\rho)$ is a $k$-size sample compression
scheme for $\F$
if
$\kappa$ is a $k$-compression and
%there is a $k$
%such that
for all
$h^*\in\F$
and all
$S=((x_1,h^{*}(x_1)),\ldots,(x_m,h^{*}(y_m)))$,
we have
$\hat h:=\rho(\kappa(S))$
satisfies $\hat h(x_i)=h^*(x_i)$ for all $i\in[m]$.

For real-valued functions, %rather than strictly requiring $\hat{h}(x_i) = h^{*}(x_i)$, 
we say it is a \emph{uniformly $\eps$-approximate} compression scheme if 
\[
\max_{1 \leq i \leq m} | \hat{h}(x_i) - h^*(x_i) | \leq \eps.
\]

%-PAC learner
%(only realizable)

%binary/real-valued
%-dual class
%
%-vc-dim
%
%-fat dim

%-compression scheme

%$c,c'$ etc is a universal constant

\section{Main results}
\label{sec:main-res}

Throughout the paper, we implicitly assume that all hypothesis classes are {\em admissible} in the sense
of satisfying mild measure-theoretic conditions, such as those specified in
\citet[Section 10.3.1]{MR876079}
or
\citet[Appendix C]{pollard84}.
We begin with an algorithmically efficient version
of the learner-to-compression scheme conversion in
\citet{DBLP:journals/jacm/MoranY16}:

\begin{theorem}[Efficient compression for classification]
  \label{thm:classification}
  Let $\C$ be a concept class
  over some instance space $\X$
  with VC-dimension $d$, dual VC-dimension $d^*$,
  and suppose that $\calA$ is a (proper, consistent) PAC-learner for $\C$:
  For all $0<\eps,\delta<1/2$, all $f^*\in\C$, and all distributions $D$ over $\X$,
  if $\calA$ receives $m\ge
  %O(\frac{d}{\eps}\log\frac{d}{\eps})
  m_\C(\eps,\delta)
  $ points $S=\set{x_i}$ drawn
  iid from $D$
  and labeled with $y_i=f^*(x_i)$,
  then
  $\calA$ outputs
an
$\hat f\in\C$
such that
\beq
\P_{S\sim D^m}\paren{ \P_{X\sim D}\paren{\hat f(X)\neq f^*(X)\gn S} >\eps}<\delta.
\eeq
For every such $\calA$,
there is
%then there exists
a randomized sample compression scheme for $\C$ of size
$O(k\log k)$, where $k=O(dd^*)$.
Furthermore,
on a sample of any size $m$,
the compression set may be computed in
expected
time
%this compression is achievable within runtime
\beq
O\paren{
  (m + \tlearn(cd))\log m
  +
  m \teval(cd) (d^* + \log m)
},
\eeq
%\todoi{AK$\to$Meni: isn't randomness involved in our compression -- so shouldn't this be expected runtime?}
where $\tlearn(\ell)$ is the runtime of $\calA$ to compute $\hat f$
on a sample of size $\ell$,
$\teval(\ell)$ is the runtime required to evaluate $\hat f$ on a single $x\in\X$,
and $c$ is a universal constant.
\end{theorem}
Although for our purposes the existence of a distribution-free
sample complexity
$m_\C$
is more important than its concrete form,
we may take 
%the latter may be assumed to be of order
$
m_\C(\eps,\delta)=
O(\frac{d}{\eps}\log\frac{1}{\eps} + \frac{1}{\eps}\log\frac{1}{\delta})$
%\citep{DBLP:journals/iandc/HausslerLW94}. %%% SH: that learner was improper actually, and the \frac{d}{\eps}\log\frac{1}{\delta} bound is generally not always achievable by a proper learner.
\citep{MR0474638,MR1072253}, 
known to bound the sample complexity of empirical risk minimization; 
indeed, this loses no generality, as there is a well-known efficient reduction 
from empirical risk minimization to any proper learner having a polynomial sample complexity \citep{pitt1988computational,haussler1991equivalence}. 
We allow the evaluation time of $\hat f$ to depend on the size of
the training sample in order to account for non-parametric learners, such as nearest-neighbor classifiers.
A naive implementation of the
\citet{DBLP:journals/jacm/MoranY16}
existence proof
yields a runtime of order
$
m^{cd}\tlearn(c'd)
+
m^{cd^*}$
(for some universal constants $c,c'$),
which can be doubly exponential when $d^*=2^d$;
this is
%just to build the ``game matrix'',
%on which a linear program must 
without taking into account the cost of
computing the minimax distribution on the $m^{cd}\times m$
game matrix.

Next, we extend the result in
Theorem~\ref{thm:classification}
from classification to regression:
\begin{theorem}[Efficient compression for regression]
  \label{thm:regression}
  Let $\F\subset[0,1]^\X$
  %over some instance spaec $\X$
  be a function class
  with
  $t$-fat-shattering dimension $d(t)$,
  dual $t$-fat-shattering dimension $d^*(t)$,
%  and suppose that $\calA$ is a (proper, consistent) PAC-learner for $\F$:
  and suppose that $\calA$ is an ERM (i.e., proper, consistent) learner for $\F$:
%  For all $0<\eps,\delta<1/2$, all 
For all $f^*\in\C$, and all distributions $D$ over $\X$,
  if $\calA$ receives $m$ 
%\ge
%  m_\F(\eps,\delta)
  %O(\frac{d}{\eps}\log\frac{d}{\eps})
%  $ 
points $S=\set{x_i}$ drawn
  iid from $D$
  and labeled with $y_i=f^*(x_i)$,
  then
%with some runtime complexity $\tau_\calA(m)$,
  $\calA$ outputs
%the classifier
an
$\hat f\in\F$
such that $\max_{i\in[m]}\abs{\hat f(x_i)-f^*(x_i)}=0$.
%and
%\beq
%\P_{S\sim D^m}
%\paren{
 % \E_{X\sim D}\sqprn{
 % \abs{\hat f(X)- f^*(X)}
%\gn S} >\eps
  %}
%<\delta.
%\eeq
For every such $\calA$,
there is
%then there exists
a randomized
uniformly $\eps$-approximate sample compression scheme for $\F$ of size %%% TODO: we should really be defining this *before* the theorem.
$O(k\tilde{m} \log( k \tilde{m} ) )$, where
$\tilde m=O(d(c\eps)\log(1/\eps))$
and
$k=O(d^*(c\eps)\log(d^*(c\eps)/\eps))$.
%\todoi{AK$\to$Meni: $k=O(?)$}
Furthermore,
on a sample of any size $m$,
the compression set may be computed in
expected
time
%this compression is achievable within runtime
\beq
O(m \teval(\tilde m) (k + \log m) +
\tlearn(\tilde m)\log(m)),
\eeq
%\todoi{AK$\to$Meni: give correct runtime}
where $\tlearn(\ell)$ is the runtime of $\calA$ to compute $\hat f$
on a sample of size $\ell$,
$\teval(\ell)$ is the runtime required to evaluate $\hat f$ on a single $x\in\X$,
and $c$ is a universal constant.
\end{theorem}  
%Here again the exact form of $m_{\F}$
%is not essential;
%barring measure-theoretic pathologies,
%it
%and may be taken
%to be
%$$m_\F(\eps,\delta)=
%O\paren{
%%  \frac1{\eps^2}\paren{\frac{d(\eps/5)}{\eps}\log\frac1\eps+\log\frac1\delta}
 % %
  %\frac1{\eps}\paren{d(c\eps)\log\frac1\eps+\log\frac1\delta}
  %},
%$$
%%\citep{DBLP:journals/jcss/BartlettL98}.
%where $c$ is a universal constant
%\citep[Theorem 20.10]{MR1741038}.

A key component in the above result is our construction
of a generic $(\dev,\adv)$-weak learner.
\begin{definition}
  \label{def:weak}
  For $\dev \in [0,1]$ and $\adv \in [0,1/2]$,
  we say that $f:\X\to\R$ is an  
  an \emph{$(\dev,\adv)$-weak hypothesis}
  (with respect to distribution $D$ and target $f^*\in\F$)
%is a 
%function $f$ with $P( x : |f(x) - f^{*}(x)| > \dev ) \leq \frac{1}{2} - \adv$.
if $$\P_{X\sim D}(|f(X)-f^*(X)|>\dev)\le \frac12-\adv.$$
\end{definition}
%We say that $f:\X\to\R$ is an
%$(\dev,\adv)$-weak hypothesis
%(with respect to distribution $D$ and target $f^*\in\F$)
%if $\P_{X\sim D}(|f(X)-f^*(X)|>\dev)\le \frac12-\adv$.

\begin{theorem}[Generic weak learner]
  \label{thm:gen-weak-learn}
  Let $\F\subset[0,1]^\X$ be a
  %admissible\footnote{
  %  state some standard measure-theoretic conditions}
  %\todoi{admissible=?}
  function class
  with $t$-fat-shattering dimension $d(t)$.
For some universal numerical constants $c_{1},c_{2},c_{3} \in (0,\infty)$, 
for any $\dev,\delta \in (0,1)$ and $\adv \in (0,1/4)$,
any $f^*\in\F$,
and any distribution $D$,
letting $X_{1},\ldots,X_{m}$ be drawn iid from $D$, where 
\begin{equation*}
  m = \left\lceil c_{1} \left(  d(c_{2} \dev ) \ln\!\left( \frac{c_{3}}{\dev} \right) + \ln\!\left(\frac{1}{\delta}\right) \right) \right\rceil,
\end{equation*}
with probability at least $1-\delta$, every $f \in \F$ with
$
\max_{i\in[m]}
|f(X_{i}) - f^{*}(X_{i})| = 0$
is an $(\dev,\adv)$-weak hypothesis
with respect to $D$ and $f^*$.
\end{theorem}
%\begin{proof}
%The result follows immediately from combining Theorem~\ref{thm:weak-learning-bound} and Lemma~\ref{lem:eps-gam-covering-numbers}.
%\end{proof}

In fact, our results would also allow us to use any hypothesis $f \in \F$ with 
$\max_{i \in [m]} |f(X_{i}) - f^{*}(X_{i})|$ bounded below $\dev$: for instance, 
bounded by $\dev / 2$.  This can then also be plugged into the construction of the 
compression scheme and this criterion can be used in place of consistency in Theorem~\ref{thm:regression}.

%\todoi{Mention applications: NN and BV regression}

In 
%the
%appendix,
%supplementary material,
Sections~\ref{sec:NN} and \ref{sec:BV} 
we give applications to sample compression
for nearest-neighbor
and
bounded-variation regression.

\section{Related work}
\label{sec:rel-work}

It appears that
generalization bounds
based on sample compression
were
independently discovered by
\citet{Littlestone86relatingdata} and \citet{MR1383093}
and further elaborated upon by \citet{DBLP:journals/ml/GraepelHS05};
see \citet{DBLP:journals/ml/FloydW95} for background and discussion.
A more general kind of Occam learning was discussed in \citet{MR1072253}.
Computational lower bounds on sample compression were
obtained in \citet{gottlieb2014near},
and some communication-based
lower bounds
%based on  complexity
were given in
\citet{DBLP:journals/corr/abs-1711-05893}.
%recent results may be classified into statistical and computational.
%Statistical:
%ref's from moran+yehudaioff
%in the context of metric spaces,
%universal bayes-consistency
%\citet{DBLP:conf/nips/KontorovichSW17}
%active learning
%\citet{DBLP:journals/corr/KontorovichSU16-nips}
%.
%Computational:
%\citet{DBLP:conf/nips/GottliebKN14}
%\citet{DBLP:journals/jmlr/GottliebKN17}
%Note that our results do not contradict
%\citet{DBLP:journals/corr/abs-1711-05893}
%because in the initial stage we retain $O(\log n)$ points
%and only reduce it to constant by sub-sampling.

\begin{sloppypar}
Beginning with \citet{FreundSchapire97}'s \algname{AdaBoost.R}
algorithm, there have been numerous attempts to extend
AdaBoost to the real-valued case
\citep{bertoni1997boosting,Drucker:1997:IRU:645526.657132,avnimelech1999boosting,Karakoulas00,DuffyHelmbold02,kegl2003robust,NOCK200725}
along with various theoretical and heuristic
constructions of particular weak regressors
\citep{Mason1999,MR1873328,DBLP:journals/ml/MannorM02};
see also the survey \citet{Mendes-Moreira2012}.
\end{sloppypar}

\citet[Remark 2.1]{DuffyHelmbold02}
%seem to have put their finger on the main
spell out a central
technical challenge:
no boosting algorithm can
``always force the
base regressor to output a useful function by simply modifying the distribution over the
sample''. This is because unlike a binary classifier, which localizes errors on specific examples,
a real-valued hypothesis can spread its error evenly over the entire sample,
and it will not be affected by reweighting.
The $(\dev,\adv)$-weak learner,
which has appeared, among other works, in
\citet{DBLP:journals/cpc/AnthonyBIS96,DBLP:journals/siamcomp/Simon97,avnimelech1999boosting,kegl2003robust},
%definition
gets around this difficulty
---
but provable general constructions of such learners have been lacking.
Likewise, the heart of our sample compression engine, \algname{MedBoost},
has been widely in use since \citet{FreundSchapire97} in various guises.
Our Theorem~\ref{thm:gen-weak-learn}
supplies the remaining piece of the puzzle:
{\em any} sample-consistent regressor
applied to some random sample of bounded size
yields
an $(\dev,\adv)$-weak hypothesis.
The closest analogue we were able to find was
\citet[Theorem 3]{DBLP:journals/cpc/AnthonyBIS96},
which is non-trivial only for function classes with finite pseudo-dimension,
and is inapplicable, e.g., to classes of $1$-Lipschitz or bounded variation functions.

The literature on general sample compression schemes for real-valued functions is quite sparse. 
There are
%classical
well-known
%specific
narrowly tailored
results on specifying functions or approximate versions of functions 
using a finite number of points, such as the
%well-known
classical
fact that a polynomial of degree $p$ can be perfectly 
recovered from $p+1$ points.
To our knowledge, the only \emph{general} results on sample compression for real-valued functions
(applicable to \emph{all} learnable function classes) is Theorem 4.3 of \citet*{david2016supervised}.
They propose a general technique to convert any learning algorithm achieving an arbitrary sample 
complexity $M(\eps,\delta)$ into a compression scheme of size $O(M(\eps,\delta) \log(M(\eps,\delta)))$, 
where $\delta$ may approach $1$.  However, their notion of compression scheme is significantly weaker 
than ours: namely, they allow $\hat{h} = \rho(\kappa(S))$ to satisfy merely $\frac{1}{m} \sum_{i=1}^{m} | \hat{h}(x_i) - h^{*}(x_i) | \leq \eps$, 
rather than our \emph{uniform} $\eps$-approximation requirement $\max_{1 \leq i \leq m} | \hat{h}(x_i) - h^{*}(x_i) | \leq \eps$.
In particular, in the special case of $\F$ a family of \emph{binary}-valued functions, their notion of 
sample compression does \emph{not} recover the usual notion of sample compression schemes for classification, 
whereas our uniform $\eps$-approximate compression notion \emph{does} recover it as a special case.  We therefore 
consider our notion to be a more
%-appropriate
fitting
generalization of the definition of sample compression 
to the real-valued case.

\section{Boosting Real-Valued Functions}
\label{subsec:real-boosting}

As mentioned above, the notion of a \emph{weak learner} for learning real-valued functions must be formulated carefully.
The na\"{i}ve thought that we could take any learner guaranteeing, say, absolute loss at most $\frac{1}{2}-\adv$ 
is known to not be strong enough to enable boosting to $\eps$ loss.  However, if we make the requirement too strong, 
such as in
%the work of
%\cite{adaboostr},
\citet{FreundSchapire97}
for
{\tt AdaBoost.R},
then the sample complexity of weak learning will be so high that weak learners cannot 
be expected to exist for large classes of functions.  %%% SH: We should be sure to say somewhere (either here or related work) what's wrong with their notion: it seems to require pseudo-dim number of samples.
However,
our Definition~\ref{def:weak},
%the following definition,
which has been proposed independently by \citet{DBLP:journals/siamcomp/Simon97} and \citet{kegl2003robust},
appears to yield the appropriate notion of \emph{weak learner} for boosting real-valued functions.

%\begin{definition}
%For $\dev \in [0,1]$ and $\adv \in [0,1/2]$, 
%an \emph{$(\dev,\adv)$-weak hypothesis} is a 
%function $f$ with $P( x : |f(x) - f^{*}(x)| > \dev ) \leq \frac{1}{2} - \adv$.
%\end{definition}

In the context of boosting for real-valued functions, 
the notion of an $(\dev,\adv)$-weak hypothesis plays a role analogous to 
the usual notion of a weak hypothesis in boosting for classification.
Specifically, the following boosting algorithm was proposed by \citet{kegl2003robust}.
As it will be convenient for our later results, we express
its output as a sequence
%the return value as vectors
of functions and weights; 
the boosting guarantee from \citet{kegl2003robust} applies to the weighted quantiles (and in particular, the weighted median) of these function values.

\RestyleAlgo{ruled}
\begin{algorithm}[h]
  \label{alg:medboost}
\SetAlgoLined
 \caption{\algname{MedBoost}($\{(x_i,y_i)\}_{i\in[m]}$,$T$,$\adv$,$\dev$)}
\begin{algorithmic}[1]
%\Procedure{\algname{MedBoost}}{$Z$,$T$,$\adv$}
\STATE Define $P_{0}$ as the uniform distribution over $\{1,\ldots,n\}$
\FOR {$t=0,\ldots,T$}
    \STATE Call weak learner to get $h_{t}$ and $(\dev/2,\adv)$-weak hypothesis wrt $(x_{i},y_{i})\! :\! i \!\sim\! P_{t}$\\ (repeat until it succeeds)
    \FOR {$i = 1,\ldots,m$}
      \STATE $\theta_{i}^{(t)} \gets 1 - 2 \ind[ |h_{t}(x_{i}) - y_{i}| > \dev/2 ]$
    \ENDFOR
%    \STATE $\alpha_{t} \gets {\rm argmin}_{\alpha} e^{\adv \alpha} \sum_{i=1}^{m} P_{t}(i) e^{-\alpha \theta_{i}^{(t)}}$
    \STATE $\alpha_{t} \gets \frac{1}{2} \ln\!\left( \frac{ (1-\adv) \sum_{i=1}^{m} P_{t}(i) \ind[ \theta_{i}^{(t)} = 1 ] }{ (1+\adv) \sum_{i=1}^{m} P_{t}(i) \ind[ \theta_{i}^{(t)} = -1 ] } \right)$
    \IF {$\alpha_{t} = \infty$} % h_{t} is already uniformly \dev-close to all y_i
      \STATE Return $T$ copies of $h_{t}$, and $(1,\ldots,1)$
    \ENDIF
    %\If $\alpha_{t} < 0$ %%% only happens if the weak learner fails.  so let's not worry about this case.
    \FOR {$i = 1,\ldots,m$}
      \STATE $P_{t+1}(i) \gets P_{t}(i) \frac{\exp\{-\alpha_{t}\theta_{i}^{(t)}\}}{\sum_{j=1}^{m} P_{t}(j) \exp\{-\alpha_{t}\theta_{j}^{(t)}\}}$
    \ENDFOR
\ENDFOR
\STATE Return $(h_{1},\ldots,h_{T})$ and $(\alpha_{1},\ldots,\alpha_{T})$ %corresponding to final estimator $H(x) = {\rm Median}(h_{1}(x),\ldots,h_{T}(x);\alpha_{1},\ldots,\alpha_{T})$
%\EndProcedure
\end{algorithmic}
\end{algorithm}

\begin{sloppypar}
Here we define the weighted median as 
\begin{equation*}
%$
{\rm Median}(y_{1},\ldots,y_{T};\alpha_{1},\ldots,\alpha_{T}) = \min\!\left\{ y_{j} : \frac{\sum_{t=1}^{T} \alpha_{t} \ind[ y_{j} < y_{t} ] }{\sum_{t=1}^{T} \alpha_{t}} < \frac{1}{2} \right\}.
%$
\end{equation*}
Also define the weighted \emph{quantiles}, for $\adv \in [0,1/2]$, as 
\begin{align*}
Q_{\adv}^{+}(y_{1},\ldots,y_{T};\alpha_{1},\ldots,\alpha_{T}) & = \min\!\left\{ y_{j} : \frac{\sum_{t=1}^{T} \alpha_{t} \ind[ y_{j} < y_{t} ] }{\sum_{t=1}^{T} \alpha_{t}} < \frac{1}{2} - \adv \right\}
\\ Q_{\adv}^{-}(y_{1},\ldots,y_{T};\alpha_{1},\ldots,\alpha_{T}) & = \max\!\left\{ y_{j} : \frac{\sum_{t=1}^{T} \alpha_{t} \ind[ y_{j} > y_{t} ] }{\sum_{t=1}^{T} \alpha_{t}} < \frac{1}{2} - \adv \right\},
\end{align*}
and abbreviate $Q_{\adv}^{+}(x) = Q_{\adv}^{+}(h_{1}(x),\ldots,h_{T}(x);\alpha_{1},\ldots,\alpha_{T})$ and $Q_{\adv}^{-}(x) = Q_{\adv}^{-}(h_{1}(x),\ldots,h_{T}(x);\alpha_{1},\ldots,\alpha_{T})$ 
for $h_{1},\ldots,h_{T}$ and $\alpha_{1},\ldots,\alpha_{T}$ the values returned by \algname{MedBoost}.
%The latter algorithm is reproduced as Algorithm~\ref{alg:medboost}
%in the %supplementary material.
%appendix.
\end{sloppypar}

Then \citet{kegl2003robust} proves the following result.

\begin{lemma}
\label{lem:kegl}
(\citet{kegl2003robust}) 
For a training set $Z = \{(x_{1},y_{1}),\ldots,(x_{m},y_{m})\}$ of size $m$, 
the return values of \algname{MedBoost} satisfy 
\begin{equation*}
\frac{1}{m} \sum_{i=1}^{m} \ind\!\left[ \max\!\left\{ \left| Q_{\adv/2}^{+}(x_{i}) - y_{i} \right|, \left| Q_{\adv/2}^{-}(x_{i}) - y_{i} \right| \right\} > \dev/2 \right] 
\leq \prod_{t=1}^{T} e^{\adv \alpha_{t}} \sum_{i=1}^{m} P_{t}(i) e^{-\alpha_{t} \theta_{i}^{(t)}}.
\end{equation*}
\end{lemma}

We note that, in the special case of binary classification, \algname{MedBoost}
is closely related to the well-known AdaBoost algorithm
\citep{FreundSchapire97}, 
and the above results correspond to a
%classic
standard
margin-based analysis of
\citet{MR1673273}.
%AK: did you mean this?
%\citet{adaboost-margin-analysis}.  
For our purposes, we will need the following immediate corollary of this, 
which follows from plugging in the values of $\alpha_{t}$ and using the weak learning assumption, 
which implies $\sum_{i=1}^{m} P_{t}(i) \ind[ \theta_{i}^{(t)} = 1 ] \geq \frac{1}{2} + \gamma$ for all $t$.

\begin{corollary}
\label{cor:kegl-T-size}
For $T = \Theta\!\left(\frac{1}{\adv^{2}} \ln(m)\right)$, 
every $i \in \{1,\ldots,m\}$ has 
\begin{equation*}
\max\!\left\{ \left| Q_{\adv/2}^{+}(x_{i}) - y_{i} \right|, \left| Q_{\adv/2}^{-}(x_{i}) - y_{i} \right| \right\} \leq \dev/2.
\end{equation*}
\end{corollary}

\section{The Sample Complexity of Learning Real-Valued Functions}
\label{subsec:weak-learning}

This section reveals our intention in choosing this notion of weak hypothesis, 
rather than using, say, an $\eps$-good strong learner under absolute loss.
In addition to being a strong enough notion for boosting to work, 
we show here that it is also a weak enough notion for the sample complexity 
of weak learning to be of reasonable size: namely, a size quantified by the fat-shattering dimension.
%which is an unavoidable complexity for learning real-valued functions.
This result is also relevant to an open question posed by
%Hans Simon
\citet{DBLP:journals/siamcomp/Simon97}, 
who proved a lower bound for the sample complexity of finding an $(\dev,\adv)$-weak hypothesis, 
expressed in terms of a related complexity measure, 
%also expressed in terms of the fat-shattering dimension, 
and asked whether a related upper bound might also hold.  
We establish a general upper bound here, witnessing the same dependence on the parameters $\dev$ and $\adv$
as observed in Simon's lower bound (up to a log factor) aside from a difference in the key complexity measure 
appearing in the bounds.
%However, the fat-shattering dimension is 
%a larger complexity measure than the corresponding factor in Simon's lower bound, and thus we do not 
%fully resolve the open question.

Define $\rho_{\dev}(f,g) = P_{2m}( x : |f(x) - g(x)| > \dev )$, where $P_{2m}$ is the empirical measure induced by $X_{1},\ldots,X_{2m}$ iid $P$-distributed random variables
(the $m$ data points and $m$ ghost points).
Define $N_{\dev}(\beta)$ as the $\beta$-covering numbers of $\F$
under the $\rho_{\dev}$ pseudo-metric.

\begin{theorem}
\label{thm:weak-learning-bound}
Fix any $\dev,\beta \in (0,1)$, $\alpha \in [0,1)$, and $m \in \mathbb{N}$. %and $m \geq 8/\beta$.
%For some numerical constant $c > 0$, 
For $X_{1},\ldots,X_{m}$ iid $P$-distributed, 
with probability at least $1 - \E\!\left[ N_{\dev (1-\alpha)/2}(\beta/8) \right] 2 e^{-m \beta / 96}$, 
every $f \in \F$ with $\max_{1 \leq i \leq m} |f(X_{i}) - f^{*}(X_{i})| \leq \alpha \dev$ 
satisfies $P( x : | f(x) - f^{*}(x) | > \dev ) \leq \beta$.
%is an $(\dev,\adv)$-weak hypothesis.
\end{theorem}
%The proof is deferred to the
%supplementary material.
%appendix.
\begin{proof} %[Proof of Theorem~\ref{thm:weak-learning-bound}]
%%% where did the covering numbers argument actually originate?  (Dudley? Pollard?)
This proof roughly follows the usual symmetrization argument for uniform convergence \citet{MR0288823,DBLP:journals/iandc/Haussler92}, 
with a few important modifications to account for this $(\dev,\beta)$-based criterion. %, and to use the metric $\rho_{\dev/2}$. 
%If $m < \frac{1}{2\adv^{2}}$ then the result is trivial (for $c \leq 2$), so let us suppose $m \geq \frac{1}{2\adv^{2}}$.
If $\E\!\left[ N_{\dev (1-\alpha)/2}(\beta/8) \right]$ is infinite, then the result is trivial, so let us suppose it is finite for the remainder of the proof.
Similarly, if $m < 8/\beta$, then $2 e^{-m \beta/96} > 1$ and hence the claim trivially holds, so let us suppose $m \geq 8/\beta$ for the remainder of the proof.
Without loss of generality, suppose $f^{*}(x) = 0$ everywhere and every $f \in \F$ is non-negative 
(otherwise subtract $f^{*}$ from every $f \in \F$ and redefine $\F$ as the absolute values of the differences; 
note that this transformation does not increase the value of $N_{\dev (1-\alpha)/2}(\beta/8)$ since applying this 
transformation to the original $N_{\dev (1-\alpha)/2}(\beta/8)$ functions remains a cover).

Let $X_{1},\ldots,X_{2m}$ be iid $P$-distributed.  Denote by $P_{m}$ the empirical measure induced by $X_{1},\ldots,X_{m}$, 
and by $P_{m}^{\prime}$ the empirical measure induced by $X_{m+1},\ldots,X_{2m}$.
We have 
\begin{align*}
& \P\!\left( \exists f \in \F : P_{m}^{\prime}( x : f(x) > \dev ) > \beta/2 \text{ and } P_{m}( x : f(x) \leq \alpha \dev ) = 1 \right)
\\ & \geq \P\left( \exists f \in \F : P( x : f(x) > \dev ) > \beta \text{ and } P_{m}( x : f(x) \leq \alpha \dev ) = 1 
\text{ and }  P_{m}^{\prime}( x : f(x) > \dev ) > \beta/2 \right).
\end{align*}
Denote by $A_{m}$ the event that there exists $f \in \F$ 
satisfying $P( x : f(x) > \dev ) > \beta$ and $P_{m}( x : f(x) \leq \alpha \dev ) = 1$, 
and on this event let $\tilde{f}$ denote such an $f \in \F$ (chosen solely based on $X_{1},\ldots,X_{m}$); 
when $A_{m}$ fails to hold, take $\tilde{f}$ to be some arbitrary fixed element of $\F$.
Then the expression on the right hand side above is at least as large as 
\begin{equation*}
\P\left( A_{m} \text{ and } P_{m}^{\prime}( x : \tilde{f}(x) > \dev ) > \beta/2 \right),
\end{equation*}
and noting that the event $A_{m}$ is independent of $X_{m+1},\ldots,X_{2m}$, this equals
\begin{equation}
\label{eqn:uc-factored-expression}
\E\!\left[ \ind_{A_{m}} \cdot \P\!\left( P_{m}^{\prime}( x : \tilde{f}(x) > \dev ) > \beta/2 \middle| X_{1},\ldots,X_{m} \right) \right].
\end{equation}
Then note that for any $f \in \F$ with $P( x : f(x) > \dev ) > \beta$, a Chernoff bound implies 
\begin{align*}
  &
  \P\Big(  P_{m}^{\prime}( x : f(x) > \dev ) > \beta/2 \Big) 
  \\ &
  = 1 - \P\Big( P_{m}^{\prime}( x : f(x) > \dev ) \leq \beta/2 \Big) 
\geq 1 - \exp\!\left\{ - m \beta / 8 \right\} 
%\geq 1 - e^{-1} 
\geq \frac{1}{2},
\end{align*}
where we have used the assumption that $m \geq \frac{8}{\beta}$ here.
In particular, this implies that the expression in \eqref{eqn:uc-factored-expression} is no smaller than 
$\frac{1}{2} \P(A_{m})$.
Altogether, we have established that 
\begin{align}
& \P\!\left( \exists f \in \F : P( x : f(x) > \dev ) > \beta \text{ and } P_{m}( x : f(x) \leq \alpha \dev ) =1 \right) \notag 
\\ & \leq 2 \P\!\left( \exists f \in \F : P_{m}^{\prime}( x : f(x) > \dev ) > \beta/2 \text{ and } P_{m}( x : f(x) \leq \alpha \dev ) = 1 \right). \label{eqn:uc-ghost-ub}
\end{align}

Now let $\sigma(1),\ldots,\sigma(m)$ be independent random variables (also independent of the data), with $\sigma(i) \sim {\rm Uniform}(\{i,m+i\})$,
and denote $\sigma(m+i)$ as the sole element of $\{i,m+i\} \setminus \{\sigma(i)\}$ for each $i \leq m$.
Also denote by $P_{m,\sigma}$ the empirical measure induced by $X_{\sigma(1)},\ldots,X_{\sigma(m)}$, 
and by $P_{m,\sigma}^{\prime}$ the empirical measure induced by $X_{\sigma(m+1)},\ldots,X_{\sigma(2m)}$.
%and denote $\boldsymbol{\sigma} = (\sigma(1),\ldots,\sigma(m))$.
By exchangeability of $(X_{1},\ldots,X_{2m})$, the right hand side of \eqref{eqn:uc-ghost-ub} is equal 
\begin{equation*}
\P\!\left( \exists f \in \F : P_{m,\sigma}^{\prime}( x : f(x) > \dev ) > \beta/2 \text{ and } P_{m,\sigma}( x : f(x) \leq \alpha \dev ) = 1 \right).
\end{equation*}
Now let $\hat{\F} \subseteq \F$ be a minimal subset of $\F$ such that 
$\max\limits_{f \in \F} \min\limits_{\hat{f} \in \hat{\F}} \rho_{\dev (1-\alpha)/2}(\hat{f},f) \leq \beta/8$.
The size of $\hat{\F}$ is at most $N_{\dev(1-\alpha)/2}(\beta/8)$, which is finite almost surely (since we have assumed above that its expectation is finite). %%% it's "at most" because \F was replaced by \{ |f-f*| : f \in \F \}.
Then note that (denoting by $X_{[2m]} = (X_{1},\ldots,X_{2m})$) the above expression is at most 
\begin{align}
& \P\!\left( \exists f \in \hat{\F} : P_{m,\sigma}^{\prime}( x : f(x) > \dev (1+\alpha)/2 ) > (3/8)\beta \text{ and } P_{m,\sigma}( x : f(x) > \dev (1+\alpha)/2 ) \leq \beta/8 \right) \notag 
\\ & \leq \E\!\left[ N_{\dev (1-\alpha)/2}(\beta/8) \max_{f \in \hat{\F}} \P\!\left( P_{m,\sigma}^{\prime}( x \!:\! f(x) \!>\! \dev (1+\alpha)/2 ) > (3/8)\beta \right.\right. \notag
\\ & {\hskip 4.5cm} \left. \phantom{\max_{f \in \hat{\F}}} \left. \text{ and } 
P_{m,\sigma}( x \!:\! f(x) \!>\! \dev (1+\alpha)/2 ) \leq \beta/8 \middle| X_{[2m]} \right) \right]\!. \label{eqn:uc-condition-on-data}
\end{align}
Then note that for any $f \in \F$, we have almost surely 
\begin{align*}
& \P\!\left( P_{m,\sigma}^{\prime}( x : f(x) > \dev (1+\alpha)/2 ) > (3/8)\beta \text{ and } P_{m,\sigma}( x : f(x) > \dev (1+\alpha)/2 ) \leq \beta/8 \middle| X_{[2m]} \right)
\\ & \leq \P\!\left( P_{2m}( x : f(x) > \dev (1+\alpha)/2 ) > (3/16) \beta \text{ and } P_{m,\sigma}( x : f(x) > \dev (1+\alpha)/2 ) \leq \beta/8 \middle| X_{[2m]} \right)
%\\ & \leq \exp\!\left\{ - m (3/16) \beta (1/3)^{2} / 2 \right\}
\\ & \leq \exp\!\left\{ - m \beta / 96 \right\},
\end{align*}
where the last inequality is by a Chernoff bound, which (as noted by \citet*{hoeffding}) remains valid even when sampling without replacement.
Together with \eqref{eqn:uc-ghost-ub} and \eqref{eqn:uc-condition-on-data}, we have that
\begin{align*}
& \P\!\left( \exists f \in \F : P( x : f(x) > \dev ) > \beta \text{ and } P_{m}( x : f(x) \leq \alpha \dev ) =1 \right) 
\\ & \leq 2 \E\!\left[ N_{\dev (1-\alpha)/2}(\beta/8) \right] e^{- m \beta / 96}.
\end{align*}
%Simplifying the constants in this expression yields the result.
\end{proof}

\begin{lemma}
\label{lem:eps-gam-covering-numbers}
There exist universal numerical constants $c,c^{\prime} \in (0,\infty)$ such that $\forall \dev,\beta \in (0,1)$,  
\begin{equation*}
N_{\dev}(\beta) \leq \left(\frac{2}{\dev\beta}\right)^{c d(c^{\prime} \dev \beta)},
\end{equation*}
where $d(\cdot)$ is the fat-shattering dimension.
\end{lemma}
\begin{proof}
  %Recall that
   \citet[Theorem 1]{MR1965359} establishes that 
the $\dev\beta$-covering number of $\F$ under the $L_{2}(P_{2m})$ pseudo-metric 
is at most 
\begin{equation}
\label{eqn:mendelson-vershynin-bound}
\left(\frac{2}{\dev\beta}\right)^{c d(c^{\prime} \dev \beta)}
\end{equation}
for some universal numerical constants $c,c^{\prime} \in (0,\infty)$.
Then note that for any $f,g \in \F$, Markov's
%inequality
and Jensen's inequalities imply
$\rho_{\dev}(f,g) \leq \frac{1}{\dev} \| f - g \|_{L_{1}(P_{2m})} \leq \frac{1}{\dev} \| f - g \|_{L_{2}(P_{2m})}$.
Thus, any $\dev\beta$-cover of $\F$ under $L_{2}(P_{2m})$ is also a $\beta$-cover of $\F$ under $\rho_{\dev}$,
and therefore \eqref{eqn:mendelson-vershynin-bound} is also a bound on $N_{\dev}(\beta)$.
%\mynote{SH: It would be better to have $(c/\adv)^{c d(\dev)}$, but at least we have this for now.}
\end{proof}

%%% SH: I'm pretty sure this can be improved to $(c/\adv)^{c d(\dev)}$ or at least improved to $(c/\adv\beta)^{c d(\dev)}$.
%%%       I think it's basically quantization (to get a finite number of possible y values) and then a covering numbers bound for multiclass classification based on the graph dimension.
%%%       I'm just still trying to figure out the latter (i.e., a covering numbers bound based on the graph dimension).
%It is not clear (to the authors) whether the bound on $N_{\dev}(\beta)$ in the above lemma can generally be improved.

Combining the above two results yields the following theorem.

%\noindent {\bf Theorem~\ref{thm:gen-weak-learn} (restated)}~~
\begin{theorem}
\label{thm:eps-gam-weak-sample-complexity}
For some universal numerical constants $c_{1},c_{2},c_{3} \in (0,\infty)$, 
for any $\dev,\delta,\beta \in (0,1)$ and $\alpha \in [0,1)$, % and $\adv \in (0,1/11)$, 
letting $X_{1},\ldots,X_{m}$ be iid $P$-distributed, where 
\begin{equation*}
m = \left\lceil \frac{c_{1}}{\beta} \left(  d(c_{2} \dev \beta (1-\alpha)) \ln\!\left( \frac{c_{3}}{\dev \beta (1-\alpha)} \right) + \ln\!\left(\frac{1}{\delta}\right) \right) \right\rceil, %(where $d(\cdot)$ is the fat-shattering dimension) 
\end{equation*}
with probability at least $1-\delta$, every $f \in \F$ with $\max_{i \in [m]} |f(X_{i}) - f^{*}(X_{i})| \leq \alpha \dev$
satisfies $P( x : | f(x) - f^{*}(x) | > \dev ) \leq \beta$.
\end{theorem}
\begin{proof}
The result follows immediately from combining Theorem~\ref{thm:weak-learning-bound} and Lemma~\ref{lem:eps-gam-covering-numbers}.
\end{proof}

In particular, Theorem~\ref{thm:gen-weak-learn} follows immediately from this result by taking $\beta = 1/2 - \adv$ and $\alpha = \adv/2$.

To discuss tightness of Theorem~\ref{thm:eps-gam-weak-sample-complexity}, we note that
%Hans Simon
\citet{DBLP:journals/siamcomp/Simon97} proved a sample complexity lower bound 
%on the sample complexity of finding an $(\dev,1/2-\beta)$-weak hypothesis with probability at least $1-\delta$:
for the same criterion of 
\begin{equation*}
\Omega\!\left( \frac{d^{\prime}(c \dev)}{\beta} + \frac{1}{\beta} \log \frac{1}{\delta} \right),
\end{equation*}
where $d^{\prime}(\cdot)$ is a quantity somewhat smaller than the fat-shattering dimension, 
essentially representing a fat Natarajan dimension.
Thus, aside from the differences in the complexity measure (and a logarithmic factor), 
we establish an upper bound of a similar form to Simon's lower bound.
%In our case, because we are considering these learners to be ``weak'', 
%we are concerned with the case of $\beta$ bounded away from $0$ (corresponding to $\adv$ bounded away from $1/2$), 
%in which case the above lower bound simplifies to $\Omega(d(\dev) + \log(1/\delta))$.
%
%In comparison, the upper bound established in our 
%Theorem~\ref{thm:eps-gam-weak-sample-complexity} 
%implies an upper bound 
%$O(d(c_{2} \adv \dev)\log(1/\dev\adv) + \log(1/\delta))$. 
%In the case of $\adv$ also bounded away from $0$, this simplifies to $O( d(c\dev)\log(1/\dev) + \log(1/\delta))$, 
%so that there is essentially only a logarithmic gap between our weak learning sample complexity bound 
%and the lower bound on the optimal sample complexity established by \citet{DBLP:journals/siamcomp/Simon97}.
%Determining whether this logarithmic factor can be removed is a very interesting question.
%In the special case of binary classification, we can effectively consider $\dev$ to itself be bounded away from $0$, 
%so that the upper bound simplifies to $O(d + \log(1/\delta))$, where $d$ denotes the VC dimension in that case.

\section{From Boosting to Compression}
\label{subsec:boosting-to-compression}

%%% A little intro discussion, featuring a summary of Moran-Yehudayoff's technique, and observation that it works just as well for binary classification with the ensemble returned by a boosting algorithm.
%%% But beyond this, we extend the technique to compression for real-valued functions, in this case meaning uniform $\eps$ loss.

Generally, our strategy for converting the boosting algorithm \algname{MedBoost} into a sample compression scheme of smaller size 
follows a strategy of Moran and Yehudayoff for binary classification, based on arguing that because the ensemble makes its predictions 
with a \emph{margin} (corresponding to the results on \emph{quantiles} in Corollary~\ref{cor:kegl-T-size}), 
it is possible to recover the same proximity guarantees for the predictions while using only a smaller \emph{subset} of the 
functions from the original ensemble.  Specifically, we use the following general \emph{sparsification} strategy.

For $\alpha_{1},\ldots,\alpha_{T} \in [0,1]$ with $\sum_{t=1}^{T} \alpha_{t} = 1$, 
denote by ${\rm Cat}(\alpha_{1},\ldots,\alpha_{T})$ the \emph{categorical distribution}: 
i.e., the discrete probability distribution on $\{1,\ldots,T\}$ with probability mass $\alpha_{t}$ on $t$.

\begin{algorithm}[h]
%  \caption{Compression using boosting}
  \caption{\algname{Sparsify}($\{(x_i,y_i)\}_{i\in[m]},\adv,T,n$)}
  \begin{algorithmic}[1]
%    \Procedure{\algname{Sparsify}}{$A,Y,y,n$}
%    \State define $P_0$ as the uniform distribution over Y
%    \For{$t \gets [0,\lceil 159 \cdot ln(n) \rceil]$} % as T > 158.3.. ln(n)
%      \State Sample $S_t \sim P_t^{64(2309 + 16d)}$ % as s > (2304 + 2ln(4/delta))/epsilon^2
%      \State $f_t = \mathcal{H}(S_t)$
%      \If{$M(P_t,f_t) < 7/8$} % epsilon = 1/8
%        \State go back to step 4
%      \EndIf
%      \ForAll{$i \in \{1,...,n\}$} % verify I didn't switch n and m
%      \State $P_{t+1}(i) = P_t(i) \cdot \exp\left(-0.106 \cdot \mathbbm{1}[f_t(x_i) = y_i] \right)$ % eta < 0.10629...
%      \EndFor
%    \EndFor
    \STATE Run \algname{MedBoost}($\{(x_i,y_i)\}_{i\in[m]},T,\adv,\dev$)
    \STATE Let $h_{1},\ldots,h_{T}$ and $\alpha_{1},\ldots,\alpha_{T}$ be its return values 
    \STATE Denote $\alpha_{t}^{\prime} = \alpha_{t} / \sum_{t^{\prime}=1}^{T} \alpha_{t^{\prime}}$ for each $t\in[T]$
%    \STATE Let $n = 64(2309 + 16d^*)$ % r  > (2304 + d + 2ln(4/delta))/epsilon^2
    \REPEAT
      \STATE Sample $(J_{1},\ldots,J_{n}) \sim {\rm Cat}(\alpha_{1}^{\prime},\ldots,\alpha_{T}^{\prime})^{n}$ %(where ${\rm Cat}(\ldots)$ is the \emph{categorical distribution})
      \STATE Let $F = \{ h_{J_{1}},\ldots, h_{J_{n}} \}$
%    \ForAll{$x \in Y$}
%\State For all $x \in Y$
%    \If {$|\{f \in F: f(x) = y(x)\}| / (64(2309 + 16d^*)) < 1/2$}
      \UNTIL {$\max_{1 \leq i \leq m} |\{f \in F: |f(x_{i}) - y_{i}| > \dev\}| < n/2$} 
%      \STATE Go back to step 3
%    \EndIf    
%    \EndFor
%    \State $Z := \bigcup_{f_t \in F}S_t$
%    \State encode the side information I
%    \State return (Z,I)    
      \STATE Return $F$
%    \EndProcedure
  \end{algorithmic}
\end{algorithm}

For any values $a_{1},\ldots,a_{n}$, denote the (unweighted) median 
\begin{equation*}
{\rm Med}(a_{1},\ldots,a_{n}) = {\rm Median}(a_{1},\ldots,a_{n}; 1,\ldots,1).
\end{equation*}
Our intention in dicussing the above algorithm is to argue that, for a sufficiently large choice of $n$, 
the above procedure returns a set $\{f_{1},\ldots,f_{n}\}$ such that 
\[
\forall i\in[m], | {\rm Med}(f_{1}(x_{i}),\ldots,f_{n}(x_{i})) - y_{i} | \leq \dev.
\]

We analyze this strategy separately for binary classification and real-valued functions, 
since the argument in the binary case is much simpler (and
demonstrates more directly
%shows more-evidently
the connection to the 
original argument of Moran and Yehudayoff), and also because we arrive at a tighter result for binary functions 
than for real-valued functions.

\subsection{Binary Classification}
\label{subsubsec:binary-classification-compression}

We begin with the simple observation about binary classification (i.e., where the functions in $\F$ all map into $\{0,1\}$).
The technique here is quite simple, and follows a similar line of reasoning to the 
original argument of Moran and Yehudayoff.  The argument for real-valued functions 
below will diverge from this argument in several important ways, but the high level 
ideas remain the same.

%compression algorithm for binary classification:
The compression function is essentially the one introduced by Moran and Yehudayoff, 
except applied to the classifiers produced by the above \algname{Sparsify} procedure, 
rather than a set of functions selected by a minimax distribution over all classifiers produced by $O(d)$ samples each.
The weak hypotheses in \algname{MedBoost} for binary classification can be obtained using samples of size $O(d)$.
Thus, if the \algname{Sparsify} procedure is successful in finding $n$ such classifiers whose median predictions are 
within $\dev$ of the target $y_{i}$ values for all $i$, 
then we may encode these $n$ classifiers as a compression set, 
consisting of the set of $k = O(nd)$ samples used to train these classifiers, 
together with $k \log k$ extra bits to encode the order of the samples.\footnote{In fact, 
$k \log n$ bits would suffice if the weak learner is permutation-invariant in its data set.}
%%% SH: We don't really need the order info within the n samples, nor do we need the order of the n samples relative to each other.
To obtain Theorem~\ref{thm:classification}, it then suffices to argue that $n=\Theta(d^{*})$ is a sufficient value.
The proof follows.

\begin{proof}[Proof of Theorem~\ref{thm:classification}]
%$VC(\mathcal{F}) = d^* < \infty$
Recall that $d^{*}$ bounds the VC dimension of the class of sets $\{ \{ h_{t} : t \leq T, h_{t}(x_i) = 1 \} : 1 \leq i \leq m \}$.
Thus for the iid samples $h_{J_{1}},\ldots,h_{J_{n}}$ obtained in \algname{Sparsify}, 
for $n = 64(2309 + 16d^*) > \frac{2304 + 16d^*+\log(2)}{1/8}$, % with distribution $\unif(\{f_1,...,f_{159\ln(n})\})^r$
by the
%classical
VC
uniform convergence inequality of %Vapnik and Chervonenkis 
\citet{MR0288823}, 
%\citet{vapnik2015uniform} 
%(with $\eps_3 < 1/8$ and $\delta_2 = 1/2$) 
with probability at least $1/2$ 
we get that
\begin{equation*}
  \max_{1 \leq i \leq m} \left| \left( \frac{1}{n} \sum_{j=1}^{n} h_{J_{j}}(x_i) \right)  - \left( \sum_{t=1}^{T} \alpha^{\prime} h_{t}(x_i) \right) \right| < 1/8.
%  \Rightarrow &\forall x \in Y: L_S(x) \geq L_D(x) - 1/8
%                = M(x,\bar{Q}) - 1/8 \geq 1 - 2\eps_1 - 1/8 - \Delta_T \\
%                & \qquad \qquad \qquad > 1 - 2 \cdot 1/8 - 1/8 - 1/8 = 1/2
%                .
\end{equation*}
In particular, if we choose $\adv = 1/8$, $\dev = 1$, and $T = \Theta(\log(m))$ appropriately, 
then Corollary~\ref{cor:kegl-T-size} implies that 
every $y_i = \ind\!\left[ \sum_{t=1}^{T} \alpha^{\prime} h_{t}(x_i) \geq 1/2 \right]$ 
and $\left| \frac{1}{2} - \sum_{t=1}^{T} \alpha^{\prime} h_{t}(x_i) \right| \geq 1/8$ 
so that the above event would imply 
every $y_i = \ind\!\left[ \frac{1}{n} \sum_{j=1}^{n} h_{J_{j}}(x_i) \geq 1/2 \right] = {\rm Med}(h_{J_{1}}(x_i),\ldots,h_{J_{n}}(x_i))$.
Note that the \algname{Sparsify} algorithm need only try this sampling $\log_{2}(1/\delta)$ times to find such a set of $n$ functions.
Combined with the description above (from \citealp{DBLP:journals/jacm/MoranY16}) 
of how to encode this collection of $h_{J_{i}}$ functions as a sample compression set plus side information, 
this completes the construction of the sample compression scheme.
\end{proof}

%K\'{e}gl's \algname{MedBoost} booster, which is just AdaBoost in this case, with $O(d)$ samples for the weak learners from standard PAC, along with $O(d^{*})$ subsampling concentration argument.
%\subsection{Algorithm}
%Altogether, this yields Theorem~\ref{thm:classification}.

\subsection{Real-Valued Functions}
\label{subsubsec:real-valued-compression}

%\begin{algorithm}
%  \caption{Sparsification}
%  \begin{algorithmic}[1]
%
%   need to fill this in.  (since the $\frac{1}{2} \pm O(\adv)$ quantile is within $\dev$, 
%we can subsample the ensemble according to the $\alpha_{t} / \sum_{t^{\prime}} \alpha_{t^{\prime}}$ proportions).
% \end{algorithmic}
%\end{algorithm}

Next we turn to the general case of real-valued functions (where the functions in $\F$ may generally map into $[0,1]$).
We have the following result, which says that the \algname{Sparsify} procedure can reduce the ensemble of functions 
from one with $T=O(\log(m)/\adv^2)$ functions in it, down to one with a number of functions \emph{independent of $m$}.
%which though somewhat weaker  than the analogous result obtained above for binary classification, 

\begin{theorem}
\label{thm:real-sparsification}
%With probabilility at least $3/4$, 
%we can stop subsampling with 
Choosing $$n = \Theta\!\left( \frac{1}{\adv^{2}} d^{*}(c\dev) \log^{2}( d^{*}(c\dev)/\dev ) \right)$$ 
suffices for the \algname{Sparsify} procedure to return $\{f_{1},\ldots,f_{n}\}$ 
with $$\max_{1 \leq i \leq m} | {\rm Med}(f_{1}(x_{i}),\ldots,f_{n}(x_{i})) - y_{i} | \leq \dev.$$
\end{theorem}
\begin{proof}
%For any vector $v\in\mathbb{R}^d$,
%let $v_1\le\ldots\le v_d'$ be its sorted version
%and let $\alpha_{1},\ldots,\alpha_{d}$ be non-negative values.
%define the operations
%$$
%{\rm Med}(v)
%:=\frac12\paren{h'_{\ceil{T/2}}(x)+h'_{\ceil{(T+1)/2}}(x)}
%$$
%and, for $0<\adv<1/2$, the
%%Let us define the
%$\frac12\pm\adv$ quantile
%%as follows:
%$$
%{\rm Quant}_{\frac12\pm\adv}(v)
%:=\set{v'_{\ceil{(1/2-\adv)T}},\ldots,v'_{\floor{(1/2+\adv)T}}}.
%$$
%Further, if $H=\set{h_1,\ldots,h_T}$
%is a collection of hypotheses then
%$H(x)\in\mathbb{R}^T$ denotes the vector of their evaluations at $x$.
%
%
Recall from Corollary~\ref{cor:kegl-T-size} that 
%For $T = \Theta\!\left(\frac{1}{\adv^{2}} \ln(m)\right)$, 
%every $i \in \{1,\ldots,m\}$ has 
\algname{MedBoost} returns functions $h_{1},\ldots,h_{T} \in \F$ and $\alpha_{1},\ldots,\alpha_{T} \geq 0$ 
such that $\forall i \in \{1,\ldots,m\}$, 
\begin{equation*}
\max\!\left\{ \left| Q_{\adv/2}^{+}(x_{i}) - y_{i} \right|, \left| Q_{\adv/2}^{-}(x_{i}) - y_{i} \right| \right\} \leq \dev/2,
\end{equation*}
where $\{(x_{i},y_{i})\}_{i=1}^{m}$ is the training data set.
%Our boosting algorithm produces a
%set of hypotheses $H=\set{h_1,\ldots,h_T}$ and $\alpha = \set{\alpha_{1},\ldots,\alpha_{T}}$ 
%such that %their weighted median
%$\hat f(x) =
%%{\rm Median}(h_{1}(x),\ldots,h_{T}(x))
%{\rm Med}(H(x))
%$
%satisfies $|\hat f(x)-f^*(x)|\le\dev/2$
%for all $x_i$ in the dataset
%$\set{x_1,\ldots,x_n}$.
%In fact, a stronger property holds.

We use this %stronger 
property to sparsify $h_{1},\ldots,h_{T}$ from $T=O(\log(m)/\adv^2)$
down to $k$
elements, where $k$ will depend on $\dev,\adv$, and
the dual fat-shattering
dimension of $\F$
(actually, just of $H = \{h_{1},\ldots,h_{T}\}\subseteq\F$)
--- but {\bf not} sample size $m$.

Letting $\alpha_{j}^{\prime} = \alpha_{j} / \sum_{t=1}^{T} \alpha_{t}$ for each $j \leq T$, 
we will sample $k$ hypotheses $\set{\tilde h_1,\ldots,\tilde h_k}=:\tilde H\subseteq H$ 
with each $\tilde{h}_{i} = h_{J_{i}}$, where $(J_{1},\ldots,J_{k}) \sim {\rm Cat}(\alpha_{1}^{\prime},\ldots,\alpha_{T}^{\prime})^{k}$ as in \algname{Sparsify}.
%(where ${\rm Cat}(\ldots)$ is the \emph{categorical distribution}, as in \algname{MedBoost}).
%uniformly at random from $H$.
Define a function $\hat{h}(x) = {\rm Med}(\tilde{h}_{1}(x),\ldots,\tilde{h}_{k}(x))$.  % := {\rm Median}(\tilde{h}_{1}(x),\ldots,\tilde{h}_{k}(x); 1,\ldots,1)$.
We claim that for any fixed $i\in[m]$,
with high probability
\begin{equation}
  \label{eq:Htil-med}
%|{\rm Median}(\tilde H(x_i))-f^*(x_i)|\le\dev/2.
|\hat{h}(x_i)-f^*(x_i)|\le\dev/2.
\end{equation}
\begin{sloppypar}
Indeed,
%evaluate each $h_t\in H$ at $x_i$
%and sort the list: $h'_1(x)\le\ldots\le h'_k(x)$.
partition the indices $[T]$ into the disjoint sets
\begin{align*}
%$L=\set{1,\ldots, \ceil{(1/2-\adv)T}-1}$,
L(x) &= \set{ j \in[ T] : h_{j}(x) < Q_{\adv}^{-}(h_{1}(x),\ldots,h_{T}(x);\alpha_{1},\ldots,\alpha_{T}) },\\
M(x) &= \set{ j \in[ T] : Q_{\adv}^{-}(h_{1}(x),...,h_{T}(x);\alpha_{1},...,\alpha_{T}) \leq\! h_{j}(x) \!\leq Q_{\adv}^{+}(h_{1}(x),...,h_{T}(x);\alpha_{1},...,\alpha_{T}) },\\
%$M=\set{\ceil{(1/2-\adv)T},\ldots, \floor{(1/2+\adv)T}}$,
%and
R(x) &= \set{ j \in[ T] : h_{j}(x) > Q_{\adv}^{+}(h_{1}(x),\ldots,h_{T}(x);\alpha_{1},\ldots,\alpha_{T}) }.
%$R=\set{\floor{(1/2+\adv)T},\ldots,T}$.
\end{align*}
Then %,
%given
%(\ref{eq:quantile}),
the only way
\eqref{eq:Htil-med}
can fail is if half or more indices %$j\in[T]$
$J_{1},\ldots,J_{k}$ sampled fall into $R(x_i)$ --- or if half or more fall into $L(x_i)$.
%Since the uniform distribution puts a mass of at most $1/2-\adv$
Since the sampling distribution puts mass less than $1/2-\adv$ 
on each of $R(x_i)$ and $L(x_i)$, 
Chernoff's bound puts an upper estimate of
$\exp(-2k\adv^2)$ on either event.
Hence,
\begin{equation}
  \label{eq:Htil-med-prob}
  \mathbb{P}\paren{|\hat{h}(x_i)-f^*(x_i)|>\dev/2}
  \le2\exp(-2k\adv^2).
\end{equation}
\end{sloppypar}

Next, our goal is to ensure that
with high probability,
\eqref{eq:Htil-med} holds simultaneously for all
${i\in[m]}$.
Define the map $\bxi:[m]\to\mathbb{R}^k$
by $\bxi(i)=(\tilde{h}_{1}(x_i),\ldots,\tilde{h}_{k}(x_i))$.
%Associate to each
%$x_i$ the vector
%$\bxi_i=(\tilde h_1(x_i),\ldots,\tilde h_k(x_i))\in\mathbb{R}^k$.
Let
%$\hat{\Xi} \subseteq\set{\xi_i:i\in[n]}=:\Xi$
$G \subseteq[m]$
be a minimal subset of $[m]$ such that 
$$\max\limits_{i \in [m]} \min\limits_{j \in {G}}
\nrm{\bxi(i)-\bxi(j)}_\infty\le\dev/2
.
$$
This is just a minimal $\ell_\infty$ covering of $[m]$.
Then
%\footnote{
%  abusing the Median notation; must define first for vectors and then for sets.
%Another abuse of notation is $f^*(\hat\bxi)$, fix later.
%}
\begin{align*}
  &\P\paren{
    \exists
    i
    \in[m]
    :
    |{\rm Med}(\bxi(i))-f^*(x_i)|>\dev
  }\le\\
  &\sum_{j\in G}
  \P\paren{
    \exists i
    :
    |{\rm Med}(\bxi(i))-f^*(x_i)|>\dev , \nrm{\bxi(i)-\bxi(j)}_\infty\le\dev/2
  }
  %\P\paren{i:\nrm{\bxi(i)-\bxi(j)}_\infty\le\dev/2}
  \le\\
  &
  \sum_{j\in G}
  \P\paren{
    |{\rm Med}(\bxi(j))-f^*(x_j)|>\dev /2
  }
  \le 2N_\infty([m],\dev/2)\exp(-2k\adv^2),
\end{align*}
where $N_\infty([m],\dev/2)$ is the $\dev /2$-covering number (under $\ell_\infty$)
of $[m]$,
and we used the fact that
$$|{\rm Med}(\bxi(i))-{\rm Med}(\bxi(j))|\le\nrm{\bxi(i)-\bxi(j)}_\infty.$$
Finally, to bound $N_\infty([m],\dev/2)$, note that
$\bxi$
embeds
$[m]$
%is embedded
into the dual
%function
class
%to
$\F^*$.
Thus, we may apply
the bound in \cite[Display (1.4)]{MR2247969}:
%Vershynin's bound\footnote{
%  Theorem 2.8 in
%  \url{http://maths-people.anu.edu.au/~mendelso/papers/summer02.pdf};
%  still trying to locate a published reference.}:
$$
\log N_\infty([m],\dev/2)\le C d^*(c\dev)\log^2(k/\dev),
$$
where $C,c$ are universal constants and $d^*(\cdot)$
is the dual fat-shattering dimension of $\F$.
It now only remains to choose a $k$
that makes
$\exp\paren{C d^*(c\dev)\log^2(k/\dev)-2k\adv^2}$
as small as desired.
\end{proof}

To establish Theorem~\ref{thm:regression}, 
we use the weak learner from above, with the booster \algname{MedBoost} from K\'{e}gl, and then apply the \algname{Sparsify} procedure. 
Combining the corresponding theorems, together with the same technique for converting to a compression scheme 
discussed above for classification (i.e., encoding the functions with the set of training examples they were obtained from, plus extra bits 
to record the order and which examples which weak hypothesis was obtained by training on), 
this immediately yields the result claimed in Theorem~\ref{thm:regression}, 
which represents our main new result for sample compression of general families of real-valued functions.

%%% SH: It's probably best if we do eventually write in a formal version of the proof, given that this is our main result.
%\begin{proof}[Proof of Theorem~\ref{thm:regression}]
%... (TODO: still need to fill in the formal details)
%\end{proof}

%%%%Meni's stuff:
\input{supl}

\subsubsection*{Acknowledgments}
% acknowledgments go at the end of the paper. Do not include
% acknowledgments in the anonymized submission, only in the final paper.
We thank Shay Moran and Roi Livni for insightful conversations.

\bibliographystyle{plainnat}
\bibliography{general}

\end{document}

%% file: supl.tex
\section{Sample compression for BV functions}
\label{sec:BV}
\newcommand{\gammat}{t}
The function class $\mathrm{BV}(v)$ consists of all $f:[0,1]\to\R$ for which
\beq
V(f):=\sup_{n\in\N}\sup_{0=x_0<x_1<\ldots<x_n=1}\sum_{i=1}^{n-1}|f(x_{i+1})-f(x_i)| \le v.
\eeq
It is known
\citep[Theorem 11.12]{MR1741038} that $d_{\mathrm{BV}(v)}(t)=1+\floor{v/(2t)}$.
In Theorem~\ref{thm:bv-dual} below, we show that the dual class has
$d^*_{\mathrm{BV}(v)}(t)=\Theta\left(\log(v/t)\right)$.
\citet{DBLP:journals/iandc/Long04} presented an efficient, proper, consistent learner
for the class $\F=\mathrm{BV}(1)$ with range restricted to $[0,1]$,
with sample complexity
$m_\F(\eps,\delta)=O(\frac1\eps\log\frac1\delta)$.
Combined with Theorem~\ref{thm:regression}, this yields
\begin{corollary}
  Let $\F=\mathrm{BV}(1)\cap[0,1]^{[0,1]}$
  be the class
  %of function
  $f:[0,1] \to [0,1]$ with $V(f)\le1$.
  Then the proper, consistent learner $\calL$
  of \citet{DBLP:journals/iandc/Long04},
  with target generalization error $\eps$,
  %enables
  admits
  a sample compression scheme of size $O(k\log k)$, where
  %Let S be a sample of $n$ labeled points
  %$(x_1,f(x_1)),...,x_n,f(x_n)) \in [0,1] \times [0,1]$ for some unknown
  %$f \in F_1$
  %There exist a $k\log{k}$-compression scheme for
  \[k = \bigO{ \uf{\eps} \log^2 \uf{\eps} \cdot \log \left(\uf{\eps} \log \uf{\eps}\right)  }.\]
The compression set is computable in expected runtime
%  which expected run-time is in order of
  \[
    \bigO{  n \uf{\eps^{3.38}} \log^{3.38} \uf{\eps}
      \left(
        \log n + \log \uf{\eps} \log
        \left(
          \uf{\eps} \log \uf{\eps}
        \right)
        \right) }
    .
  \]
\hide{
  if we denote $\uf{\widetilde{\eps}} := \uf{\eps} \log \uf{\eps}$
  we get that
  \[k \in \bigO{ \uf{\widetilde{\eps}} \log \uf{\widetilde{\eps}} \log \uf{\eps}}\]
  And the expected run-time is
  \[\bigO{ \frac{n}{\widetilde{\eps}^{3.38}}
      \left(
        \log n + \log \uf{\eps} \log \uf{\widetilde{\eps}}
        \right) }\]
  }
\end{corollary}

\hide{
We show two specific applications of the above algorithm in cases where the dual fat-shattering dimension
is relatively low.

Recall the bounded-variation property for $f:[0,1] \to [0,1]$ functions

\begin{definition}
  definition of bounded variation
\end{definition}

Let $F_V$ be the class of function $f:[0,1] \to [0,1]$ whose variation is bounded by $V$.
and let $F_V^*$ be the dual function class, namely, $[0,1]$ equipped with the evaluation defined by
\[\forall x \in [0,1],f \in F_V : x(f) := f(x)\]
}

The remainder of this section is devoted to proving
\begin{theorem}
  \label{thm:bv-dual}
  For $\F=\mathrm{BV}(v)$
  and $t<v$,
  we have $d^*_\F(t)=\Theta\left(\log(v/t)\right)$.
  \hide{
  Let $F_V$ be the class of function $f:[0,1] \to [0,1]$ whose variation is bounded by $V$.
  and let $F_V^*$ be the dual function class, namely, $[0,1]$ equipped with the evaluation defined by
  \[\forall x \in [0,1],f \in F_V : x(f) := f(x)\]
  we get that
  \[VC(F_V^*) \in \Theta\left(\log_2(V/\gammat)\right)\]
  }
\end{theorem}

%To prove this we first define some notions and prove a result regrading class-coding-theory problems
First, we define some preliminary notions:

\begin{definition}
  For a binary $m \times n$ matrix $M$, define
  %  Let M be a binary matrix of size $m \times n$
  \beq
    V(M, i) &:=& \sum_{j=1}^m \ind[M_{j,i} \neq M_{j+1,i}],\\
    G(M) &:=& \sum_{i=1}^n V(M,i),\\
    V(M) &:=& \max_{i \in [n]}V(M,i).
\eeq
\end{definition}

\begin{lemma} \label{lemma:1}
  Let $M$ be a binary $2^n \times n$ matrix.
  If
for each
$b \in \{0,1\}^n$
there is a row $j$ in $M$
equal to $b$,
%\  \exists j \in [2^n] \ s.t. \ M_{j,*} = b$
%  Then
then
\[V(M) \geq \frac{2^n}{n}.\]
In particular,
for at least one row $i$,
we have
%there is a row $i$ s.t.
$V(M,i) \geq 2^n/n$.
\end{lemma}

\begin{proof}
  Let $M$ be a $2^n \times n$
  binary
  such that
for each
$b \in \{0,1\}^n$
there is a row $j$ in $M$
equal to $b$.  
%
%  s.t.  \[\forall b \in \{0,1\}^n \  \exists j \in [2^n] \ s.t. \ M_{j,*} = b\]
Given $M$'s dimensions,
%size
every
$b \in \{0,1\}^n$ appears exactly in one row of $M$,
and hence
%so obviously
the minimal Hamming distance between two rows is 1.
Summing over the $2^n-1$ adjacent row pairs, we have
%  There are $2^n$ rows the total Hamming-distance or ``variation'' through the matrix is at least $2^n-1$.
\[G(M) = \sum_{i=1}^n V(M,i) = \sum_{i=1}^n\sum_{j=1}^m \ind[M_{j,i} \neq M_{j+1,i}] \geq 2^n-1,\]
which averages to
%  so the average of the variation is
\[\frac{1}{n} \sum_{i=1}^n V(M,i) = \frac{G(M)}{n} \geq \frac{2^n-1}{n} .\]
By
%  We get by
the pigeon-hole principle,
there must be
%that at there is
a row $j \in [n]$ for which
$V(M,i) \geq \frac{2^n-1}{n}$,
which implies
%  By definition this means that
$V(M)  \geq \frac{2^n-1}{n}$.
\end{proof}
%\mynote{it seems that we can omit the V(M) notion}

We split the proof of Theorem~\ref{thm:bv-dual} into two estimates:

%Now back to the proof of \hyperref[theorem 1]{Theorem 12}. We split the proof into two lemmas

\begin{lemma} \label{lem:bv-ub}
  %$VC(\F^*) \leq 2 log_2(v/\gammat)$
  For $\F=\mathrm{BV}(v)$ and $t<v$,
  $d^*_\F(t)\le2\log_2(v/\gammat)$.
\end{lemma}

\begin{lemma} \label{lem:bv-lb}
  %$VC(\F^*) \geq log_2(v/\gammat)$
  For $\F=\mathrm{BV}(v)$ and $4t<v$,
  $d^*_\F(t)\ge\floor{\log_2(v/\gammat)}$.
\end{lemma}

\begin{proof}[Proof of Lemma~\ref{lem:bv-ub}]
  Let $\set{f_1,\ldots,f_n}\subset \F$ be a set of functions that are $\gammat$-shattered by $\F^*$.
  In other words, there is an $r\in\R^n$
  such that for each
%
%let $r_1,...,r_n$ bet the ``witnesses'' of the shattering.
  %For every
  $b \in \{0,1\}^n$
there is an $x_b
\in\F^*$
such that
%meaning that
\[\forall i \in [n] ,
  x_b(f_i)
  \begin{cases}
    \geq r_i + \gammat, & b_i = 1 \\
    \leq r_i - \gammat, & b_i = 0
  \end{cases}
  .
\]

Let us order the $x_b$s by
magnitude
$x_1<x_2<\ldots<x_{2^n}$,
denoting this sequence by
$(x_i)_{i=1}^{2^n}$.
Let $M \in \{0,1\}^{2^n \times n}$ be a matrix whose $i$th row is $b_j$,
the latter ordered arbitrarily.

%Put $B = \{x_b : b \in \{0,1\}^n\}$ and let $\{x_j\}_{j=1}^{2^n}$ the elements of $B$ ordered by the natural order on th reals.
%On the same way denote  $\{b_j\}_{j=1}^{2^n}$ the binary vectors of length $n$ which order is induced by $B$.

%Constant a matrix $M \in \{0,1\}^{2^n \times n}$ whose $j$th row is $b_j$.

By Lemma~\ref{lemma:1}, there is $i \in [n]$ s.t. 
\[\sum_{j=1}^{2^n} \ind[M(j,i) \neq M(j+1,i)] \geq \frac{2^n}{n}.\]
Note that if
$M(j,i) \neq M(j+1,i)$ shattering implies that
\[x_j(f_i) \geq r_i + \gammat \text{ and } x_{j+1}(f_i) \leq r_i - \gammat\]
or
\[x_j(f_i) \leq r_i - \gammat \text{ and } x_{j+1}(f_i) \geq r_i + \gammat;\]
either way,
%we get that
\[\abs*{f_i(x_j) - f_i(x_{j+1})} = \abs*{x_j(f_i) - x_{j+1}(f_i)} \geq 2\gammat.\]
So
for the function
$f_i$,
%the following holds
we have
\beq
  \sum_{j=1}^{2^n} \abs*{f_i(x_j) - f_i(x_{j+1})} = \sum_{j=1}^{2^n} \abs*{x_j(f_i) - x_{j+1}(f_i)} 
  \geq \sum_{j=1}^{2^n} \ind[b_{j_i} \neq b_{{j+1}_i} \cdot 2\gammat 
  \geq \frac{2^n}{n} \cdot 2\gammat.
\eeq
As $\{x_j\}_{j=1}^{2^n}$ is a partition of $[0,1]$ we get
%that
\[v \geq \sum_{j=1}^{2^n} \abs*{f_i(x_j) - f_i(x_{j+1})} \geq  \frac{\gammat 2^{n+1}}{n} \geq \gammat 2^{n/2}\]
and hence
%So it must hold that
\[v/\gammat \geq 2^{n/2}\]
\[\Rightarrow 2\log_2(v/\gammat) \geq n.\]
\end{proof}

\begin{proof}[Proof of Lemma~\ref{lem:bv-lb}]
%  \mynote{Round the value to get a proper integers}
  We construct a set of $n = \floor{\log_2(v/\gammat)}$ functions that are $t$-shattered by $\F^*$.
  First, we build a balanced Gray code \citep{flahive2007balancing} with $n$ bits,
which we arrange into the rows of $M$.
%  $M = [b_1^T,...,b_{2^n}^T]$.  
Divide the unit interval into $2^n
%= v/\gammat
$ segments and define, for each $j\in[2^n]$,
\[
%\forall j \in [v/\gammat] :
x_j := \frac{j}{2^n}
%= \frac{j\gammat}{v}
.
\]
  Define the functions $f_1,\ldots,,f_{\floor{\log_2(v/\gammat)}}$ as follows:
  \[f_i(x_j) =
    \begin{cases}
      \gammat, &
      M(j,i)
      %{b_j}_i
      = 1
      \\
      -\gammat, &
      %{b_j}_i
      M(j,i)
      = 0
    \end{cases}.
  \]
We claim that each
%Next we prove that indeed
$
%\forall i:
f_i \in \F$.
Since
%As the code
$M$ is balanced Gray code,
%by definition
%  \\ \mynote{we assume that $v/\gammat \geq 4$}
\[V(M) = \frac{2^n}{n} \le \frac{v}{\gammat \log_2(v/\gammat)} \leq \frac{v}{2\gammat}.\]
Hence, for each
%So for every function
$f_i$,
we have
%it holds that
  \[V(f_i) \leq 2 \gammat V(M,i) \leq 2 \gammat \frac{v}{2\gammat} = .v\]
Next, we show
%Finally we prove
that this set is shattered by $\F^*$.
Fix the trivial offest $r_1=...=r_n = 0$
For every $b \in \{0,1\}^n$ there is a $j \in [2^n]$ s.t. $b = b_i$.
By construction,
%it holds that
for every
%index
$i \in [n]$,
we have
  \[x_j(f_i) = f_i(x_j) =
    \begin{cases}
      \gammat \geq r_i + \gammat, & M(j,i)
              %{b_j}_i
              = 1 \\
              -\gammat \leq r_i - \gammat, & M(j,i)
              %{b_j}_i
              = 0
    \end{cases}.
  \]
\end{proof}

%By \hyperref[lemma:2]{Lemma 15} and \hyperref[lemma:3]{Lemma 16} we conclude that
%\[VC(\F^*) \in \Theta\left(\log_2(v/\gammat)\right)\]

%Combining this result with our prior result and using Long's algorithm Long's ***cite***
%we can apply our compression scheme framework to get
%a compression scheme for the $1$-bounded-variation function class

\section{Sample compression for nearest-neighbor regression}
\label{sec:NN}
Let $(\X,\rho)$ be a metric space
and define, for $L\ge0$, the collection $\F_L$ of all $f:\X\to[0,1]$ satisfying
$$ |f(x)-f(x')|\le L\rho(x,x');$$
these are the $L$-Lipschitz functions.
\citet{GottliebKK17_IEEE} showed that
\beq
d_{\F_L}(t) = O\paren{ \ceil{L\diam(X)/t}^{\ddim(\X)}},
\eeq
where $\diam(\X)$ is the diameter and $\ddim$ is the {\em doubling dimension}, defined therein.
The proof is achieved via a packing argument, which also shows that the estimate is tight.
Below we show that
$  d^*_{\F_L}(t) =
\Theta(\log\left({M}(\X,2\gammat/L)\right))$,
where $M(\X,\cdot)$ is the packing number of $(\X,\rho)$.
Applying this to the efficient
nearest-neighbor regressor\footnote{
  In fact, the technical machinery in
  \citet{DBLP:journals/tit/GottliebKK17}
  was aimed at achieving {\em approximate} Lipschitz-extension,
  so as to gain a considerable runtime speedup. An {\em exact} Lipschitz extension
  is much simpler to achieve. It is more computationally costly but still polynomial-time in sample size.
  }
of
\citet{DBLP:journals/tit/GottliebKK17},
we obtain
%Using this result, Gottlieb et al's bounds ***cite*** and the nearest-neighbor algorithm as a Lipschitz-regressor
%we get that 
\begin{corollary}
  Let $(\X,\rho)$ be a metric space with hypothesis class $\F_L$,
  and let $\calL$ be a consistent, proper learner for $\F_L$ with target generalization error $\eps$.
%  Let $(\mathcal{X}, \rho)$ be a metric space. Let S be a sample of $n$ labeled points
%  $(x_1,f(x_1)),...,x_n,f(x_n)) \in \mathcal{X} \times \mathbb{R}$ for some unknown
%  $f \in \left( \F_L(\mathcal{X}) \right) $
Then $\calL$ admits a compression scheme of size $O(k\log k)$, where
%  There exist a $k\log{k}$-compression scheme for
  \[k = \bigO{ D(\eps) \log \uf{\eps} \cdot \log D(\eps)
      \log \left(\uf{\eps} \log D(\eps) \right) }\]
%
%  which expected run-time is in order of
%  \[
%    \bigO{ n D(\eps) \log \uf{\eps}
%      \left( \log n + \log D(\eps) \log \left(\uf{\eps} \log D(\eps) \right) \right) }
%  \]
and
%  For
\[D(\eps) = \ceil{\frac{L\diam(\X)}{\eps}}^{\ddim(\X)}.\]
\hide{
  If we fix $diam(\mathcal{X}) = 1$ and $L = 1$ we get that
  \[k \in \bigO{ \frac{d}{\eps^d} \log^2 \uf{\eps}
      \log \left(\frac{d}{\eps} \log \uf{\eps}  \right) }\]
  and
  \[
    \bigO{ \frac{n}{\eps^d} \log \uf{\eps}
      \left(
        \log n + d \log \uf{\eps} \log 
        \left(
          \frac{d}{\eps} \log \uf{\eps}
        \right)
      \right) }
  \]

  for $d := ddpim(\mathcal{X})$
 } 
\end{corollary}

\hide{
Recall that
For a metric space $(\mathcal {X}, \rho)$
Let $F_L(\mathcal{X})$ be the class of real-valued functions $f:\mathcal{X} \to [0,1]$ who are L-Lipschitz.
and let $F_L^*$ be the dual function class, namely, $[0,1]$ equipped with the evaluation defined by
\[\forall x \in \mathcal{X},f \in F_L : x(f) := f(x)\]

Lipschitz function are useful in on general metric spaces as they provide the most natural
regularized function class, when the Lipschitz-parameter L is used as the regularizator.
}

\hide{
As before, in order to apply our framework we first need to bound the dual-dimension
\begin{lemma}\label{lem:lip-vc}
  For a metric space $(\mathcal {X}, \rho)$
  Let $\F_L(\mathcal{X})$ be the class of real-valued functions $f:\mathcal{X} \to [0,1]$ who are L-Lipschitz.
  and let $\F_L^*$ be the dual function class, namely, $[0,1]$ equipped with the evaluation defined by
  \[\forall x \in \mathcal{X},f \in \F_L : x(f) := f(x)\] we get that
  \[VC(\F_L^*) = log_2\left(\mathcal{M}(\mathcal{X},2\gammat/L)\right)\]
\end{lemma}

We split the proof into two lemmas
}

We now prove our estimate on the dual fat-shattering dimension of $\F$:

\begin{lemma}
  %\label{lemma:1}
  For $\F=\F_L$,
  $d^*_\F(t)
  \leq \log_2\left(\mathcal{M}(\mathcal{X},2\gammat/L)\right)$.
%  $VC(\F_L^*) \leq log_2\left(\mathcal{M}(\mathcal{X},2\gammat/L)\right)$
\end{lemma}
\begin{proof}
  Let $\set{f_1,\ldots,f_n}\subset \F_L$ a set that is $\gammat$-shattered by $\F^*_L$.
  For $b \neq b' \in \{0,1\}^n$,
  let $i$ be the first index
  for which $b_i \neq b'_i$,
  say, $b_i = 1\neq0= b'$.
  By shattering,
  %As the set is $\gammat$-shattered by $F^*_L$
  there are points $x_b,x_{b'} \in \F^*_L$ such that
  $x_b(f_i) \geq r_i + \gammat$ and
  $x_{b'}(f_i) \leq r_i - \gammat$,
  whence
%  So we get that
  \[f_i(x_b) - f_i(x_{b'}) \geq 2\gammat\]
  and
%  and by the Lipschitz property we get that
  \[L  \rho(x_b,x_{b'}) \geq  f_i(x_b) - f_i(x_{b'}) \geq 2\gammat.\]
  It follows that for
%  So to conclude
  %\[\forall
  $b \neq b' \in \{0,1\}^n$,
we have
$ \rho(x_b,x_{b'}) \geq 2\gammat / L$.
  Denoting by ${M}(\X, \eps)$ the $\eps$-packing number of $\X$,
 we get
  \[2^n = |\{x_b \mid b \in \{0,1\}^n\}| \leq \mathcal{M}(\mathcal{X}, 2\gammat / L). \]
%  \[\Rightarrow n \leq log_2\left(\mathcal{M}(\mathcal{X},2\gammat/L)\right)\]
\end{proof}

\begin{lemma}
  %\label{lemma:2}
  For $\F=\F_L$ and $t<L$,
  $d^*_\F(t)  \geq \log_2\left(\mathcal{M}(\mathcal{X},2\gammat/L)\right)$.
%  $VpC(\F_L^*) \geq log_2\left(\mathcal{M}(\mathcal{X},2\gammat/L)\right)$
\end{lemma}
\begin{proof}
  Let $S=\{x_1,...,x_m\} \subseteq \X$
  be a maximal
  $2\gammat/L$-packing
of $\X$.
%  of size ${M}(\X,2\gammat/L)$].
%  which is a minimal 
%  [so it is of size $m = 
Suppose that
%Let
$c: S \to \{0,1\}^{\floor{\log_2m}}$
%some arbitrary injection.
is one-to-one.
Define the set of function $F = \{f_1,\ldots,f_{\floor{\log_2(m)}}\} \subseteq \F_L$ by
  \[
    f_i(x_j) =
    \begin{cases}
      \gammat, & c(x_j)_i = 1 \\
      -\gammat, & c(x_j)_i = 0
    \end{cases}.
  \]
  For every $f \in F$ and every two points $x,x' \in S$ it holds that
  \[\abs{f(x) - f(x')} \leq 2\gammat = L \cdot 2\gammat / L \leq L  \rho(x,x').\]
  This set of functions is $\gammat$-shattered by $S$
  %(and from that also by $\mathcal{X}$)
  and is of size
  $\floor{\log_2m} = \floor{\log_2\left(\mathcal{M}(\mathcal{X},2\gammat/L)\right)}$.
\end{proof}